\documentclass{article}
\usepackage{ijcai21}

\usepackage{times}
\usepackage{soul}
\usepackage{url}
\usepackage[hidelinks]{hyperref}
\usepackage[utf8]{inputenc}
\usepackage[small]{caption}
\usepackage{graphicx}
\usepackage{amsmath}
\usepackage{amsthm}
\usepackage{booktabs}
\usepackage{algorithm}
\usepackage{algorithmic}

\newtheorem{theorem}{Theorem}

\title{Boosting Variational Inference With Locally Adaptive Step-Sizes}
\author{}

\author{
Gideon Dresdner,\textsuperscript{1}\and
Saurav Shekhar\textsuperscript{1}\and\\
Fabian Pedregosa\textsuperscript{2}\and
Francesco Locatello\textsuperscript{1}\and\\
Gunnar~R\"atsch\textsuperscript{1}
\affiliations
\textsuperscript{1}
Dept. for Computer Science, ETH Zurich, Universitätsstrasse 6, 8092 Zurich, Switzerland\\
\textsuperscript{2}
Google Research
\emails
dgideon@ethz.ch
}
\newcommand{\citet}{\cite}
\newcommand{\citep}{\cite}

\usepackage{amsmath, amsthm, amssymb, amsfonts, bm}
\usepackage{mathtools}
\usepackage{nicefrac}
\usepackage{enumitem}
\usepackage{algorithm}
\usepackage{algorithmic}
\usepackage{xcolor, color, colortbl}
\definecolor{mydarkblue}{rgb}{0,0.08,0.45}

\usepackage{booktabs}
\usepackage{xspace}

\newcommand{\hide}[1]{}

\newcommand{\cQ}{\mathcal Q}

\newcommand{\Q}{\mathcal Q}

\newcommand{\grad}{\nabla}
\DeclareMathOperator*{\Exp}{\mathbb{E}}

\DeclareMathOperator*{\argmin}{arg\,min}
\DeclareMathOperator*{\argmax}{arg\,max}
\DeclareMathOperator{\conv}{conv}

\DeclarePairedDelimiter\theset{\{}{\}}

\newcommand{\D}{\mathcal{D}}
\newcommand{\domain}{\cQ}
\newcommand{\Cf}{{\mathcal C}_{\cQ}}

\newcommand{\CfQ}{C_{f,\bar{\cQ}}}

\newcommand{\dkl}{D^{KL}}

\newcommand{\innerp}[1]{\langle #1\rangle}

\newcommand\given[1][]{#1\vert}

\usepackage{xcolor, color, colortbl}
\definecolor{Gray}{gray}{0.92}
\definecolor{mydarkgreen}{RGB}{39,130,67}
\newcommand{\green}{\color{mydarkgreen}}
\definecolor{mydarkred}{RGB}{192,47,25}
\newcommand{\red}{\color{mydarkred}}
\usepackage{pifont}
\newcommand{\cmark}{\ding{51}}
\newcommand{\xmark}{\ding{55}}
\newcommand{\greencheck}{\green\large\cmark}
\newcommand{\redx}{\red\large\xmark}

\newcommand{\bX}{\mathbf{X}}
\newcommand{\by}{\mathbf{y}}
\newcommand{\bw}{\mathbf{w}}

\newcommand{\defeq}{\stackrel{\text{def}}{=}}
\newcommand{\dual}[1]{#1^{*}}
\newcommand{\indicator}{\imath}
\newcommand{\linearconvergence}{\mathcal{O}(1/t)}

\newcommand{\operatortextsc}[1]{\operatorname{\textsc{#1}}}

\newcommand{\gammamax}{\gamma^{max}}

\usepackage{scalerel,stackengine}
\stackMath
\newcommand\reallywidehat[1]{%
\savestack{\tmpbox}{\stretchto{%
  \scaleto{%
    \scalerel*[\widthof{\ensuremath{#1}}]{\kern-.6pt\bigwedge\kern-.6pt}%
    {\rule[-\textheight/2]{1ex}{\textheight}}%
  }{\textheight}%
}{0.5ex}}%
\stackon[1pt]{#1}{\tmpbox}%
}
\usepackage{fancyhdr} 
\usepackage{cancel}

\makeatletter
\newtheorem*{rep@theorem}{\rep@title}
\newcommand{\newreptheorem}[2]{%
\newenvironment{rep#1}[1]{%
 \def\rep@title{#2 \ref{##1}}%
 \begin{rep@theorem}}%
 {\end{rep@theorem}}}
\makeatother

\newreptheorem{theorem}{Theorem}
\newtheorem{lemma}{Lemma}
\newreptheorem{lemma}{Lemma}
\newtheorem{remark}{Theorem}
\newreptheorem{remark}{Theorem}
\newtheorem{proposition}{Theorem}
\newreptheorem{proposition}{Theorem}

\newcommand\markerlessfootnote[1]{%
  \begingroup
  \renewcommand\thefootnote{}\footnote{#1}%
  \addtocounter{footnote}{-1}%
  \endgroup
}

\begin{document}

\maketitle

\vspace{-3mm}
\begin{abstract}
Variational Inference makes a trade-off between the capacity of the variational family and the tractability of finding an approximate posterior distribution. Instead, Boosting Variational Inference allows practitioners to obtain increasingly good posterior approximations by spending more compute. The main obstacle to widespread adoption of Boosting Variational Inference is the amount of resources necessary to improve over a strong Variational Inference baseline. In our work, we trace this limitation back to the global curvature of the KL-divergence. We characterize how the global curvature impacts time and memory consumption, address the problem with the notion of local curvature, and provide a novel approximate backtracking algorithm for estimating local curvature. We give new theoretical convergence rates for our algorithms and provide experimental validation on synthetic and real-world datasets. 
\end{abstract}

\vspace{-3.5mm}
\section{Introduction}
\vspace{-.5mm}

The central problem of Bayesian inference is to estimate the posterior distribution $p(z\given X)$ of hidden variables~$z$, given observations $X$, a likelihood model $p(X\given z)$, and a prior distribution $p(z)$.
The Variational Inference (VI) approach 
\citep{jordan1999introduction,Blei2017review}
consists in finding the best approximation in Kullback-Leibler (KL)-divergence to the posterior from within a family of tractable densities $\Q$. This is posed as the following optimization problem:
\begin{equation}
  \label{eq:vanillavi}
  \argmin_{q\in\Q} \left\{ \dkl(q) \defeq \int q(z)\ln\frac{q(z)}{p(z|X)} dz \right\}
\end{equation}
We occasionally abuse the notation for $\dkl$: When the target distribution is omitted, as above, it is understood that the target is the true posterior distribution $p(z|X)$.

There is a trade-off between the quality of the approximation and the difficulty of the optimization problem.
While a richer family may yield a better approximation of the posterior, finding such a solution requires solving a more complex optimization problem.
A growing body of recent work addresses this trade-off by specifying variational families that are richer but still tractable~\citep{RezendeM15,Saeedi2017VPA,Salimans2015MarkovCM,Saxena2017Variational,Cranko2018Mar}. However, once the VI solver has converged, one cannot spend more compute to improve the approximation.%
If the approximation is too poor to be useful, it must be abandoned and the inference procedure restarted with a richer variational family.

The recent line of work in Boosting VI takes a different approach.
Instead of specifying a richer variational family, Boosting VI greedily constructs one using mixtures of densities from a simpler base family~\citep{Guo:2016tg,Miller:2016vt,LocKhaGhoRat18,locatello2018boosting,Cranko2018Mar}.
The key idea is that one can iteratively build a better approximation to the target posterior by fitting the residual parts which are not yet well-approximated.

Despite advances in making boosting agnostic to the choice of the variational family \citep{locatello2018boosting}, this line of work has fallen short of its potential. %
The reason for this is that Boosting VI does not reliably improve the variational approximation in a reasonable number of iterations \citep{Guo:2016tg,LocKhaGhoRat18,locatello2018boosting}.

In this work, we present a new technique for determining the mixture weights of Boosting VI algorithms which improves the variational approximation in a realistic number of iterations.
As we shall see, the previous convergence rates depend on two terms: a term depending on the \textit{global} curvature of the KL and the initial error.
Practitioners often focus on decreasing the latter term, but the first one can be arbitrarily large without imposing extra assumptions.%

We are able to improve the dependency on the curvature in the convergence rate by tuning the mixture weights %
according to a quadratic function satisfying a sufficient decrease condition, i.e.\ is sufficiently tight on the KL divergence objective \citep{Pedregosa2018Jun}.

\looseness=-1%
In the black-box VI setting, checking for exact upper bounds is not feasible due to sampling errors. Therefore, we consider the case where the estimate of the bound is \textit{inexact}. Using this approximate local upper-bound, we develop a fast and memory efficient away-step black-box Boosting VI algorithm.

Our {\bfseries main contributions} can be summarized as follows:
\vspace{-1mm}
\begin{enumerate}[leftmargin=*]
\item 
We introduce an \textit{approximate} sufficient decrease condition and prove that the resulting backtracking algorithm converges with a rate of $\mathcal O(1/t)$ with an improved dependency on the curvature constant.
\item
We develop an away-step Boosting VI algorithm that relies on our approximate backtracking algorithm.
This enables Boosting VI methods to selectively downweight previously seen components to obtain sparser solutions, thus reducing overall memory costs.
\item
We present empirical evidence demonstrating that our method is both faster and more robust than existing methods.
The adaptive methods also yield more parsimonious variational approximations than previous techniques.
\end{enumerate}

\vspace{-2mm}
\section{Related work}

\begin{table*}[t]
\centering
\small
\begin{tabular}{r|cccccc}
    & \shortstack{conv.\ to true post.} & \shortstack{black-box} & \shortstack{ada.\ weight tuning}
  & \shortstack{KL obj.} & \shortstack{flexible num.\ of comps.} & \shortstack{gen.\ post.}\\
  \cline{1-7}
  \citep{Guo:2016tg} &           \redx &        \redx &        \redx &        \greencheck &    \greencheck &  \greencheck\\
  \cellcolor{Gray}\citep{Miller:2016vt}\cellcolor{Gray}&    \cellcolor{Gray}     \redx &    \cellcolor{Gray}    \redx & \cellcolor{Gray}       \greencheck & \cellcolor{Gray} \greencheck & \cellcolor{Gray}   \greencheck & \cellcolor{Gray} \greencheck\\
  \citep{locatello2018boosting}& {\greencheck}$^\ast$ &  \greencheck &  \redx &        \greencheck &    \greencheck &  \greencheck\\
  \cellcolor{Gray}\citep{LocKhaGhoRat18} & \cellcolor{Gray}      {\greencheck}$^\ast$  & \cellcolor{Gray}\redx & \cellcolor{Gray}       \redx & \cellcolor{Gray}       \greencheck  & \cellcolor{Gray}  \greencheck & \cellcolor{Gray} \greencheck\\
  \citep{campbell2019universal} &\greencheck &  \redx &        \greencheck &  \redx &          \greencheck &  \greencheck\\
  \cellcolor{Gray}\citep{Cranko2018Mar} & \cellcolor{Gray}       \greencheck & \cellcolor{Gray} \redx & \cellcolor{Gray}       \greencheck & \cellcolor{Gray} \greencheck &  \cellcolor{Gray}  \greencheck& \cellcolor{Gray}  \redx\\
  \citep{lin2019fast} &  \redx &        \redx &        \greencheck &  \greencheck &    \redx &        \greencheck\\
  \textbf{\emph{This work}} \cellcolor{Gray}&  \cellcolor{Gray}  \greencheck  & \cellcolor{Gray}\greencheck & \cellcolor{Gray} \greencheck & \cellcolor{Gray} \greencheck & \cellcolor{Gray}   \greencheck & \cellcolor{Gray} \greencheck\\
\end{tabular}
\caption{Comparison to previous work.
\emph{Conv.\ to true post.}: the paper provides asymptotic convergence guarantees ($\ast$) under mild conditions on the variational family $\cQ$ such as clipped tails and initialization in the neighborhood of the optimum.
\emph{black-box}: agnostic to the form of the variational family or target distribution.
\emph{ada.\ weight tuning}: provides methods for tuning the mixture weights based on the quality of the components.
\emph{KL obj.}: minimizes the traditional KL-divergence VI objective.
\emph{flexible num.\ comps.}: does not require the user to specify the number of components in advance.
\emph{gen.\ post.}: the approximate posterior can be sampled from and not merely used to evaluate the probability density.
}\label{tab:rel_works}
\end{table*}

We refer to~\citet{Blei2017review} for a review of Variational Inference (VI). Our focus is to use boosting to increase the complexity of a density, similar to the goal of Normalizing Flows~\citep{RezendeM15}, MCMC-VI hybrid methods~\citep{Saeedi2017VPA,Salimans2015MarkovCM}, distribution transformations~\citep{Saxena2017Variational}, and boosted density estimation
\citep{Cranko2018Mar,locatello2018clustering}.
Our approach is in line with several previous ones using mixtures of distributions to improve the expressiveness of the variational approximation~\citep{Jaakkola98Mixtures,Jerfel2017thesis} but goes further to draw connections to the optimization literature.
While our method does not leverage classical weak learners as in \citep{Cranko2018Mar}, it does return a solution which is sampleable and is therefore more amenable to downstream Bayesian analysis.

While boosting has been well studied in other settings~\citep{meir2003introduction}, it has only recently been applied to the problem of VI. Related works of~\citet{Guo:2016tg,Miller:2016vt} developed the algorithmic framework and conjectured a convergence rate of $\linearconvergence$.
Later, \citet{LocKhaGhoRat18} identified sufficient conditions for convergence and provided explicit constants to the~$O(1/t)$~rate.
They based their analysis on the smoothness of the KL-divergence when using carefully constructed variational base families which are restrictive in practice.

In \citet{locatello2018boosting}, these assumptions were reduced to a simple entropy regularizer which allows for a black-box implementation of the boosting subroutine. %
The promise of their work is to make Boosting VI useful in practice while retaining an $\mathcal O(1/t)$ convergence rate. %

Recent work by \citet{campbell2019universal} and \citet{lin2019fast} also explore the relationship between Boosting VI and curvature.
In this paper, we present our unique view on curvature which does not require sacrificing the KL-divergence for a smooth objective as in \cite{campbell2019universal} or fixing the number of components in the mixture \citet{lin2019fast}.
Finally, note that~\citet{LocKhaGhoRat18} also suggests performing line-search on a quadratic approximation of the KL-divergence. Crucially, they suggest using a \textit{global} approximation which renders their step-size arbitrarily small. See Table~\ref{tab:rel_works} for comparison to prior work.

Backtracking line-search and related variants are well-understood for deterministic objectives with projectionable constraints \citep{Boyd:2004uz}. However, backtracking techniques have only recently been applied to Frank-Wolfe by \citet{Pedregosa2018Jun}. Our work lies at the intersection of these developments in VI and backtracking line-search.

\section{Boosting Variational Inference}

Boosting Variational Inference (VI) aims to
solve an expanded version of the problem defined in Equation~\eqref{eq:vanillavi}
by optimizing over the convex hull of $\cQ$ defined as,
\begin{equation*}
  \conv(\Q) \defeq \left\{
    \textstyle\sum_i \alpha_i q_i\mid q_i\in\cQ,\,\textstyle\sum_i \alpha_i = 1,\, \alpha_i > 0
  \right\}
\end{equation*}
\looseness=-1%
Boosting VI algorithms take a greedy, two-step approach to solving this problem.
At each iteration, first, the posterior residual $\nicefrac{p_X}{q_t}$ is approximated with an element of $\Q$ and added to the list of components.
Then, the weights of the mixture are updated. Previous research on Boosting VI has focused on the first step --- selecting the best fitting density of the residual --- whereas our work takes into consideration the second step --- adjusting the weights. As we shall see, step-size choice has a significant impact on both the constant in the convergence rate as well as the observed speed in practice.

\textbf{Selecting the next component:} Greedily approximating the residual $\nicefrac{p_X}{q_t}$ is equivalent to solving a linear minimization problem as first realized by \citet{Guo:2016tg} and later formalized by \citet{LocKhaGhoRat18}.
Our contribution --- presented in Section~\ref{sec:local-boosting}~---~can be combined with any of these approaches.

Without imposing additional constraints on the boosted Variational Inference problem, the greedy subproblem has degenerate solutions, i.e.\ Dirac delta located at the maximum of the residual \citep{LocKhaGhoRat18,locatello2018boosting,Guo:2016tg,Miller:2016vt}. This is the central challenge addressed by existing work on Boosting Variational Inference.
The approach taken in  
\citet{locatello2018boosting} is to use a generic entropy regularizer as the additional constraint. This conveniently reframes the subproblem as another KL-minimization problem which can then be fed into existing black-box VI solvers.

In their work, the greedy step can be formulated as a constrained linear minimization problem over the variational family,
\begin{equation} \label{def:lmo}
    \argmin_{\substack{s\in\cQ\\H(s)\geq -  M}} 
    \innerp{s,\grad\dkl(q_t)}
\end{equation}
where $H$ is the differential entropy functional. This results in a modified overall Variational Inference objective over the entropy constrained mixtures, rather than all of $\conv(\cQ)$:
\begin{equation}
    \argmin_{\conv(\overline\cQ)} \dkl(q)
\end{equation}
where $\overline{\mathcal{Q}} \defeq \theset{s\in\mathcal Q\mid H(s) \geq -M}$.

By reformulating the differential entropy constraint in Equation~\eqref{def:lmo} using a Lagrange multiplier and then setting the multiplier to one, one arrives at a convenient form for the greedy subproblem (Alg.~\ref{alg:main} line 3):
\begin{equation}
\label{eq:bbbvi}
    \argmin_{s\in\cQ}
    \dkl\left(s\,\Big\|\,\frac{p_X}{q_t}\right) = \argmin_{s\in\cQ} \Exp_s[\ln s] - \Exp_s [\ln \frac{p_X}{q_t}]
\end{equation}
Intuitively, this objective encourages the next component to be close to the target posterior $\ln p_X$ while simultaneously being different from current iterate $\ln q_t$ and also being non-degenerate via the negative entropy term $\ln s$.

\textbf{Predefined step-size:} To update the mixture,~\citet{locatello2018boosting} take a convex combination between the current approximation and the next component (Alg.~\ref{alg:main} line~7) with a \textit{predefined} step-size of $\gamma_t=\frac{2}{t+2}$:
\begin{align}\label{eq:predefined}
    q_{t+1} = \left(1-\frac{2}{t+2}\right)q_t + \frac{2}{t+2} s_t
\end{align}
where $v_t$ is set to $q_t$.
\section{Local Boosting Variational Inference}\label{sec:local-boosting}
We now describe our main algorithmic contribution.

\smallskip
\noindent
\textbf{Notation:}
We view expectations as a special case of functional inner-products. Given two functionals, $a,b: z\mapsto \mathbb R$, their inner-product is $\innerp{a,b}\defeq\int a(z)b(z) dz$. Practically, we only encounter these integrals when one of the arguments is a density that can be sampled from thus allowing us to use Monte-Carlo: $\innerp{a,b}\approx \widehat{\innerp{a,b}} \defeq \frac{1}{n}\sum_{i=1}^k b(\zeta_i)$ where $\zeta_i\sim a$.

\smallskip
\noindent
\textbf{Assumption:}
We assume that for all $\varepsilon > 0$ there exists an $n\in\mathbb N$ number of Monte-Carlo samples such that for all $s\in\mathcal Q$ and $q\in\conv(\mathcal Q)$, the Monte-Carlo approximation $\reallywidehat{\innerp{s,\grad\dkl(q)}}$ is $\varepsilon$-close to the true value:
\begin{equation}
    | \innerp{s,\grad\dkl(q)} - \reallywidehat{\innerp{s,\grad\dkl(q)}}| \leq \varepsilon
    \label{eq:assumption}
\end{equation}
This assumption states that we can approximate the value of the objective in Equation~\eqref{def:lmo} up to a predefined tolerance.

\smallskip

Now, suppose we are at iteration $t$ of the boosting algorithm (Alg.~\ref{alg:main}). $q_t$ is the current variational approximation containing at most $t$ components. The next component $s_t$ is provided by line 3. $v_t$ is then returned from the corrective components procedure (described in Sec.~\ref{sec:corrective} and App.~\ref{app:fwvariants}).
Let $d_t = s_t - v_t$ be the update direction.

Our goal is to solve the following one dimensional problem,
\begin{equation}
  \label{eq:stepsizeproblem}
   \gamma_t \in \argmin_{\gamma\in [0,1]} \dkl(q_t + \gamma d_t)
\end{equation}
Then we can set $q_{t+1} = q_t + \gamma_t d_t$ as described in line~7 of Algorithm~\ref{alg:main}.
Solving this problem --- often termed ``line-search'' --- may be hard when no closed-form solution is available e.g.\ in the case of black-box VI~\citep{Locatello:2017tq,LocKhaGhoRat18,locatello2018boosting,Pedregosa2018Jun}. In practice, general approaches such as gradient descent struggle to handle the changing curvature of the KL-divergence throughout $\conv(\overline\cQ)$ (c.f.\ Sec.~\ref{sec:experiments}).

 Instead, \citet{Pedregosa2018Jun} uses so-called Deminov-Rubinov line-search. Rather than solving the line-search problem directly, their technique defines a surrogate objective:
\begin{align}
\label{eq:defQt}
    Q_t(\gamma,C)
    &\defeq \dkl(q_t) + \gamma\innerp{d_t,\nabla\dkl(q_t)} + \frac{C\gamma^2}{2}
\end{align}
Importantly, there exists $C_t>0$ such that for all $\gamma\in [0,1]$, $Q_t(\gamma,C_t)$ is a valid upper bound on the line-search problem (Eq.~\eqref{eq:stepsizeproblem}).
In this case, we say that $C_t$ satisfies the sufficient decrease condition:
\begin{align}~\label{eq:decrease_condidtion}
    \dkl(q_t + \gamma d_t) \leq Q_t(\gamma, C_t)
\end{align}
Essentially, we are bounding the first-order Taylor expansion of the KL-divergence at $q_t + \gamma d_t$ (as in Eq.~\eqref{eq:stepsizeproblem}).

Unlike the finite dimensional setting described in \citet{Pedregosa2018Jun}, in black-box VI we are unable to validate the sufficient decrease condition directly because we can only approximate $Q_t$ within a precision of $\varepsilon_t$,
\begin{equation}
|Q_t(\gamma, C) - \widehat Q_t(\gamma, C)| \leq \varepsilon_t
\end{equation}
where $\widehat Q_t$ is the Monte-Carlo approximation to $Q_t$.
This leads us to define an \emph{approximate} sufficient decrease condition (c.f.\ Appendix~\ref{sec:assumption-consequences}):
\begin{equation}\label{eq:defQt-hat}
\hspace{-1.1mm}
    \widehat Q_t'(C,\gamma) \defeq \widehat\dkl(q_t) + \gamma\reallywidehat{\innerp{\grad\dkl(q_t), d_t}} + \frac{C}{2}\gamma^2 \mathbf{\color{red}+ 2\bm\varepsilon_t}
\end{equation}
where $\varepsilon_t = \mathcal{O}(1/t^2)$. If we assume that the number of samples is large enough, then we can guarantee that
\begin{equation}
    \dkl(q_t + \gamma d_t) \leq Q_t(C,\gamma) \leq \widehat Q_t'(C,\gamma)
\end{equation}
Intuitively, we are adding an offset of $\varepsilon_t$ to compensate for errors in the Monte-Carlo approximation of $Q_t$. As we will see in Section~\ref{sec:analysis}, to obtain an overall convergence rate we require that $\varepsilon_t$ decreases at each iteration. Equivalently, this requires increasing the number of samples at each step of the algorithm.

\vspace{-1mm}
Suppose that, for some $C_t$, $\widehat Q_t'$ satisfies the approximate sufficient decrease condition.
Then, $\widehat Q_t'(\gamma,C_t)$ can easily be minimized with respect to $\gamma_t$ by setting the derivative with respect to $\gamma$ equal to zero and solving:
\vspace{-.5mm}
\begin{equation}
  \label{eq:ls_closedform}
  \gamma_t
  = \min\left\{ -\frac{\reallywidehat{\innerp{\nabla \dkl(q_t), d_t}}}{C_t},\gammamax_t\right\}
\end{equation}
\vspace{-.5mm}
\noindent
where $\gammamax_t\in (0,1]$ depends on the corrective algorithm variant (c.f.\ Sec.~\ref{sec:corrective}).
Equation~\eqref{eq:ls_closedform} shows that $C_t$ and $\gamma_t$ are inversely proportional.
Intuitively, the new component should be aggressively favored when it is highly correlated with the gradient of the KL-divergence since this gradient is the optimal decrease direction.
We want to take advantage of these low curvature opportunities to take more aggressive steps towards improving the posterior approximation.

\setlength{\textfloatsep}{2pt}
\setlength{\floatsep}{5pt}
\begin{algorithm}[bt]
  \caption{Template for Boosting VI algorithms}
  \label{alg:main}
\begin{algorithmic}[1]
\STATE{\bfseries Input:} $q_0\in\cQ,\,C_{-1}>0$
  \FOR{$t=0, 1 \ldots $}
  \STATE $s_t = \argmin_{s\in\cQ} \dkl(s\,\|\,\frac{p_X}{q_t})$\hfill\COMMENT{next component}
  \STATE $v_t,\gammamax_t = \operatortextsc{correct\_components}(s_t,q_t)$\label{line:vt}
  \STATE $\gamma_t, C_t = \operatortextsc{find\_step\_size}(q_t,s_t-v_t,C_{t-1},\gammamax_t)$
  \STATE{Update:} $q_{t+1} = q_t + \gamma_t (s_t-v_t)$
  \ENDFOR
\end{algorithmic}
\end{algorithm}
\begin{algorithm}[bt]
\caption{Find step-size with approximate backtracking}\label{alg:approxback}
\begin{algorithmic}[1]
\FUNCTION{$\operatortextsc{find\_step\_size}(q_t,d_t,C,\gammamax_t$)}
\STATE{Choose:} $\tau > 1$, $\eta \leq 1$, $\varepsilon_0 > 0$, $\textsc{imax} \in \mathbb{Z}_+$
\STATE{Let:} $g_t = -\innerp{\grad \dkl(q_t),d_t}$
\STATE{Set:} $C \leftarrow C/\eta$, $\gamma = \min\{g_t/C,1\}$, $i = 0$
\WHILE{$\widehat\dkl(q_t + \gamma d_t) > \widehat Q_t'(\gamma,C)$}
\IF{$i > \textsc{imax}$}
\STATE{\bfseries return} $\frac{2}{t+2},C$\hfill\COMMENT{predefined step-size}
\ENDIF
\STATE $C \leftarrow \tau C$
\STATE $\gamma = \min\{g_t/C,\gammamax_t\}$
\STATE $i \leftarrow i+1$
\ENDWHILE
\STATE{\bfseries return} $\gamma,C$\hfill\COMMENT{adaptive step-size}
\ENDFUNCTION
\end{algorithmic}
\end{algorithm}

Therefore, our goal is to find a $C_t$ which satisfies the approximate sufficient decrease conditions of Equation~\ref{eq:decrease_condidtion} while also being as small as possible.
This is achieved using approximate backtracking on $C$ (Alg.~\ref{alg:approxback}).

The cost of this procedure is the cost of estimating the sufficient decrease condition for each proposal of~$C_t$. The Monte-Carlo estimates used to compute $\widehat Q_t'$ can be reused but the approximation to $\dkl(q_t+\gamma d_t)$ must be re-estimated.

\looseness=-1%
Observe that in order to guarantee convergence, there must exist a global bound on $C_t$.
This quantity is known as the \emph{global} curvature and is defined directly as the supremum over all possible $C_t$'s \citep{Jaggi:2013wg}:
\vspace{-.5mm}
\begin{equation}\label{def:Cf}
    C_{\cQ} \defeq \sup_{\substack{s\in\cQ\\q\in\conv(\cQ)\\
    \gamma\in[0,1]\\y = q + \gamma(s-q)}} \frac{2}{\gamma^2} \dkl(y\,\|\,q)
\end{equation}
Section~\ref{sec:analysis_discussion} provides theoretical results clarifying the relationship between~$C_t$~and~$C_\cQ$.

We make a slight modification to the algorithm to circumvent this problem, summarized in lines~(6-8)~of~Algorithm~\ref{alg:approxback}.
When the adaptive loop fails to find an approximation within $\textsc{imax}$ number of steps, it exits and simply performs a predefined step-size update.
We find that this trick allows us to quickly escape from curved regions of the space. After just one or two predefined update steps, we can efficiently estimate the curvature and apply the approximate backtracking procedure. See Appendix~\ref{app:results} for results on early termination.

\vspace{-1mm}
\subsection{Correcting the current mixture}\label{sec:corrective}
Not only does our work analyze and solve the problem of varying curvature in Boosting VI, it also enables the use of more sophisticated, corrective variants of the boosting algorithm. Corrective variants aim to address the problem of downweighting or removing suboptimally chosen components. This is important in the approximate regime of VI where most quantities are estimated using Monte-Carlo. However, it is impossible to apply these corrective methods to boosting without a step-size estimation method such as we describe in Section~\ref{sec:local-boosting}.

In the optimization literature, there are two corrective methods (c.f.\ App.~\ref{app:fwvariants}).
Either one of these variants can be substituted for the~$\operatortextsc{correct\_components}$~procedure. 
Both corrective methods begin by searching for the worst previous component, $\overline{v}$~in the sense of most closely aligning with the positive gradient of the current approximation:
\begin{equation}
    \overline{v} = \argmax_{v\in\mathcal S_t} \left\{\innerp{v,\grad \dkl(q_t)} = \Exp_v\ln\frac{p_X}{q_t} \right\}
\end{equation}
where $\mathcal S_t = \{s_1,s_2,\ldots,s_k\}$ is the current set of components. $\bar v$ is found by estimating each element of $\mathcal S_t$ using Monte-Carlo samples and selecting the argmax. %
\smallskip

\noindent
\textbf{Implementation:}
Using the work of \citet{locatello2018boosting}, we perform the greedy step using existing Variational Inference techniques. For example, if $\cQ$ is a reparameterizable family of densities, then the reparameterization trick can be used in conjugation with gradient descent methods on the parameters of $s\in\cQ$. All the quantities necessary to compute $\widehat Q_t'$, $\gamma_t$, and $\bar v$, are estimated using Monte-Carlo.

\vspace{-2mm}
\section{Convergence Analysis}\label{sec:analysis}
\vspace{-.5mm}

\begin{figure*}[t]
\vspace{1mm}
  \centering
  {
  \begin{tabular}{lccc}
  & Train LL & Test AUROC & Time (in s)\\
  \midrule
  AdaAFW      & {\bf -0.669 $\pm$ 5.390e-04}& 0.787 $\pm$ 4.599e-03 & 48.240 $\pm$ 2.384e+01 \\
  AdaPFW      & -0.672 $\pm$ 6.340e-04& {\bf 0.791 $\pm$ 2.398e-03} & 107.247 $\pm$ 1.675e+02 \\
  AdaFW       & -0.671 $\pm$ 9.700e-04& 0.789 $\pm$ 7.985e-03 & 40.870 $\pm$ 1.647e+01 \\
  Predefined$^*$  & -0.676 $\pm$ 7.435e-04& 0.788 $\pm$ 7.401e-03 & {\bf 7.296 $\pm$ 1.777e+00} \\
  Line-search & {\bf -0.669 $\pm$ 1.011e-03} & {\bf 0.791 $\pm$ 7.514e-03} & 265.608 $\pm$ 1.655e+02 \\
  \end{tabular}
  \captionof{table}{
  \small
  Comparison of different step-size selection methods on \textsc{ChemReact} dataset. Adaptive variants have comparable AUROC values while taking less time and having less variance across multiple runs. (*) Predefined is the method proposed by \protect\cite{locatello2018boosting}.
  }\label{tab:blr}
  }
{
\begin{tabular}{lccc}
  & Train LL & Test AUROC & Time (in s)\\
  \midrule
  AdaAFW      & {\bf -0.169 $\pm$ 1.111e-03}& {\bf 0.859 $\pm$ 2.565e-03} & 74.634 $\pm$ 3.369e+01 \\
  AdaPFW      & -0.172 $\pm$ 1.361e-03& {\bf 0.857 $\pm$ 1.011e-03} & 149.721 $\pm$ 1.088e+02 \\
  AdaFW       & {\bf-0.170 $\pm$ 1.774e-03} & {\bf 0.859 $\pm$ 3.672e-03} & 54.334 $\pm$ 2.632e+01 \\
  Predefined$^*$  & -0.181 $\pm$ 2.983e-03& 0.853 $\pm$ 3.693e-03 & {\bf 21.369 $\pm$ 9.631e+00} \\
  Line-search & -0.181 $\pm$ 2.651e-03& 0.854 $\pm$ 3.473e-03 & 145.725 $\pm$ 1.347e+02 \\
\end{tabular}
  \captionof{table}{
  \small
  Comparison of different step-size selection methods on \textsc{eICU} dataset. Adaptive away-steps variant gives the best test AUROC as well as training log-likelihood. (*) Predefined is the method proposed by \protect\citet{locatello2018boosting}.} \label{tab:blr_eicu}
}
\end{figure*}

The following theorem shows a convergence rate for our algorithm. It extends the work of \citep{Pedregosa2018Jun} to the case of Variational Inference
\begin{theorem}\label{thm:convergence}
  Let $q_t$ be the $t$-th iterate generated by Algorithm~\ref{alg:main}. Let $\varepsilon_t = \frac{\varepsilon_0}{t^2}$ bound the Monte-Carlo approximation error, described in Equation~\eqref{eq:assumption}, with some initial approximation error $\varepsilon_0>0$.
  Let $\overline C_t \defeq \frac{1}{t}\sum_{i=0}^{t-1} C_t$ be the average of the curvature constant estimates.
  Then we have:
  \begin{align*}
    \dkl(q_t) - \dkl(q^\ast) \leq 
    \frac{4(1-\delta)}{t\delta(t\delta+1)}E_0
    + \frac{2\overline{C_t}}{\delta(t\delta + 1)}
    + \frac{2\varepsilon_0}{t}
  \end{align*}
  where $E_0 \defeq \dkl(q_0) - \psi(\nabla \dkl(q_0))$ is the initialization error ($\psi$ denotes the dual objective) and $\delta>0$ bounds the error of estimating the greedy subproblem defined in Equation~\eqref{eq:bbbvi}. See Appendix~\ref{app:proof} for the proof.
\end{theorem}
It is clear that the Monte-Carlo estimates must be increasingly accurate at each iteration to guarantee convergence. This can be achieved by increasing the number of samples.

The average $\overline C_t$ will be kept much smaller than the global curvature $C_\mathcal{Q}$.
Previous results give rates in terms of global curvature \cite{locatello2018boosting,LocKhaGhoRat18,Guo:2016tg,campbell2019universal}. Without making additional assumptions, the global curvature is in principle unbounded. This explains why the number of iterations~$t$~must be large before observing the decrease in the error expected from prior work.

\subsection{Discussion}\label{sec:analysis_discussion}

The authors of \citet{locatello2018boosting,LocKhaGhoRat18} also provide a convergence rate for their algorithm. However, in practice this rate is never achieved. This is due to a dependence on the global curvature (Eq.~\eqref{def:Cf}) of the KL-divergence which can be huge in certain areas of $\conv(\cQ)$. Since the inner loop of the algorithm is essentially a full run of Variational Inference, it is impossible to run it enough times to beat the global curvature and decrease the initial error as~$\linearconvergence$.

We reduce the rate by focusing on the curvature constant which, as shown by Equation~\eqref{eq:ls_closedform}, is directly related to the estimation of the mixture weights. In practice, our approximate backtracking approach on the local curvature is a viable middle ground between exact line-search, which is expensive for the KL divergence, and predefined step-size which requires an infeasibly large number of iterations.

\begin{figure*}[ht]
  \vspace{-1.5mm}
  \includegraphics[width=0.33\linewidth]{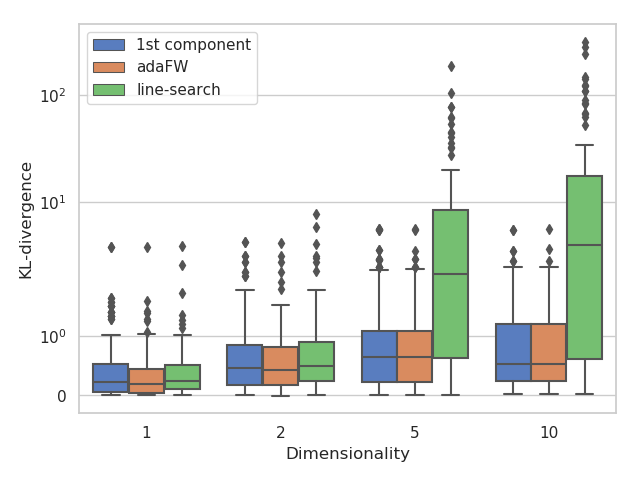}
  \includegraphics[width=0.66\linewidth]{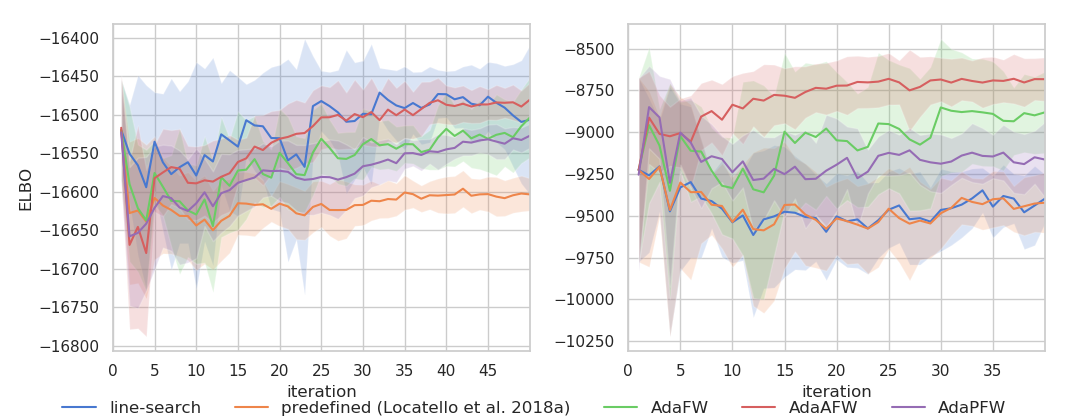}
  \captionof{figure}{\small (left) KL-Divergence of mixture to the target distribution for different step-size  variants with random LMO (lower is better). ELBO values vs Frank-Wolfe iteration for different step-size selection methods on Bayesian logistic regression task
  for \textsc{ChemReact} (center) and \textsc{eICU} (right) datasets (higher is better).
  Solid lines are mean values and shaded regions
  are standard deviations over different parameter configurations and 10 replicates.
  In both the cases, adaptive variants achieve higher ELBO and are more stable than line-search.
 }
\label{fig:rand_lmo_and_blr_elbo}
\end{figure*}

\begin{figure}[ht]
    \vspace{-5mm}
    \centering
    \includegraphics[width=.75\linewidth]{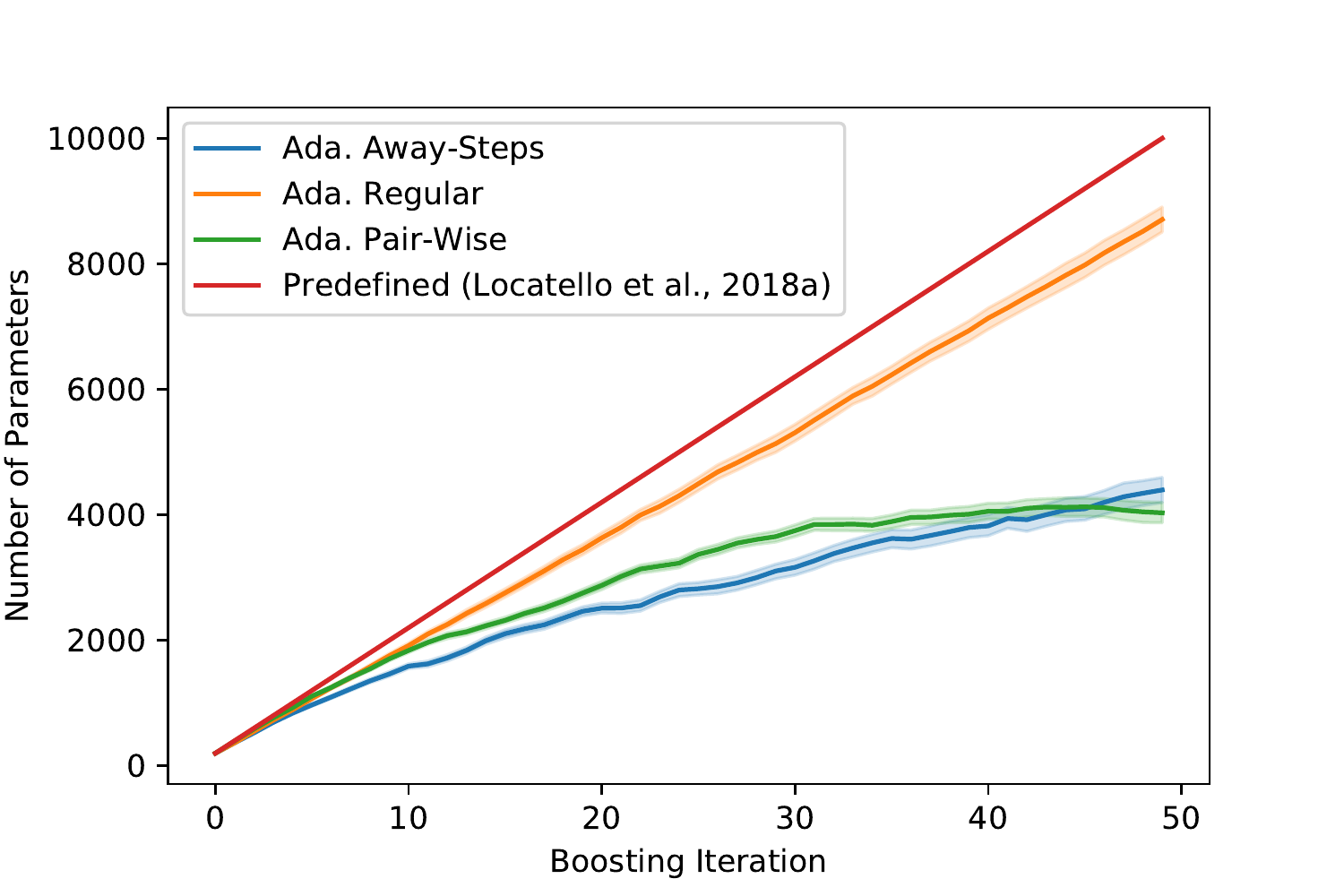}
    \caption{\small Comparing the number of parameters per iteration to previous work on \textsc{ChemReact}.
    }\label{fig:nparams}
    \label{fig:memusage_chemreact}
\end{figure}

Recognize that the backtracking approach of \citet{Pedregosa2018Jun} cannot be applied directly to the VI setting because~$Q_t$ cannot be computed exactly. 
Overestimation of~$Q_t$~due to approximation errors results in slower convergence.
However, if $C_t$ is underestimated, then the step-size will be overestimated which breaks the convergence proof of~\citet{Pedregosa2018Jun}. By introducing the parameter $\varepsilon_t$, we can provide some control over this underestimation problem as well as fully characterize the theoretical consequences.

\textbf{Limitations:}
Determining the number of samples needed to satisfy the assumption of Section \ref{sec:local-boosting}, namely, the number of samples needed to guarantee the convergence rate provided in Theorem \ref{thm:convergence}, is a long-standing open problem. There is work on this subject ranging from pure optimization \cite{paquette_stochastic_2018} to Importance Weighted Autoencoders \cite{burda2015importance}. %
In this paper, our goal is to characterize the problem theoretically in terms of convergence rates. We believe that this clarifies the ongoing gap between theory and practice. Despite this gap, we found empirically that \texttt{100} Monte-Carlo samples was sufficient. 

\section{Empirical validation}\label{sec:experiments}

\subsection{Instability of line-search}\label{subsec:syn_data}
\markerlessfootnote{Source code to reproduce our experiments is available here: \url{https://github.com/ratschlab/adaptive-stepsize-boosting-bbvi}}
We tested the behavior of gradient-based line-search methods in a well-understood synthetic setup
in which the target posterior is a mixture of two Gaussians.
We set the first variational component to be the first component of the target density
assuming a perfect first iteration.
We then compare the performance of step $t=1$ of the algorithm, with that of line-search and our adaptive method at step $t=2$, over multiple runs.
To measure overall performance, we use (approximate) KL-divergence to the true posterior.

Line-search becomes unstable as the dimensionality increases. This is because line-search oscillates between the two extreme values of zero and one, regardless of the quality of the next component. The quality of the next random component decreases with the dimensionality due to the curse of dimensionality. Even in medium dimensional cases, line-search is unable to distinguish between these cases. Results are summarized in Figure~\ref{fig:rand_lmo_and_blr_elbo} (left).

\subsection{Bayesian logistic regression} \label{subsec:blr}

We consider two real-world binary-classification tasks: predicting
the reactivity of a chemical and predicting mortality in the intensive care unit~(ICU).
For both tasks we use Bayesian logistic regression.
Bayesian logistic regression is a conditional prediction model with prior~$p(\bw)~=~\mathcal{N}(0,1)$~and conditional likelihood $p(\by|\bX) = \text{Bernoulli}(p = \text{sigmoid}(\bX^\top\bw))$.
This model is commonly used as an example of a simple model which does not have a closed-form posterior~\citep{Blei2017review}.
We set the base family to be the Laplace distributions following \citet{locatello2018boosting}.
\smallskip

\textbf{Chemical reactivity:}
For this task we used the \textsc{ChemReact}\footnote{http://komarix.org/ac/ds/} dataset which contains 26,733 chemicals, each with~100~features.
We ran our algorithm multiple times for 50 iterations.
We selected the iterate with the best median train log-likelihood over 10 replicates in the first 50 boosting iterations (Table~\ref{tab:blr}).
\smallskip

\textbf{Mortality prediction:}
For this task we used the \textsc{eICU Collaborative Research} database~\citep{Goldbergere215}.
Following \citet{fortuin2018som}, we selected a subset of the data with 71,366 patient stays and 70 relevant features ranging from age and gender to lab test results. We ran boosting for 40 iterations and, for each algorithm, chose the iteration which gave best median train log-likelihood over 10 replicates (Tab.~\ref{tab:blr_eicu}).

In both datasets, we observe that adaptive variants achieve a better ELBO and are more stable than line-search (Fig.~\ref{fig:rand_lmo_and_blr_elbo}~(center and right)). This results in better AUROC and train log-likelihood (Tab.~\ref{tab:blr},~\ref{tab:blr_eicu}).

Naturally, predefined step-size is the fastest since it simply sets the step-size to $2/(t+2)$. But, this results in suboptimal performance and unnecessarily large variational approximations. Our approach results in at least a 2x speed-up over line-search as well as better performance.

Adaptive variants are also faster than line-search (Tables~\ref{tab:blr}~and~\ref{tab:blr_eicu}). Step-size adaptivity is only 2-5 times slower than predefined step-size as opposed to line-search which is 7-39 times slower. Overall, we observe that step-size adaptivity is faster, more stable, and yields more accurate models than line-search. 

\subsection{Memory efficiency}

Our corrective methods discussed in Section~\ref{sec:corrective} not only yield superior solutions in terms of accuracy, but also yield more parsimonious models by removing previously selected components.
Figure~\ref{fig:nparams} demonstrates this behavior on the \textsc{ChemReact} dataset.
Appendix~\ref{app:results} has similar results on the \textsc{eICU} dataset.

\section{Conclusion}

In this paper, we traced the limitations of state-of-the-art boosting variational inference methods back to the global curvature of the KL-divergence. We characterized how the global curvature directly impacts both time and memory consumption in practice, addressed this problem using the notion of local curvature, and provided a novel approximate backtracking algorithm for estimating the local curvature efficiently. Our convergence rates not only provide theoretical guarantees, they also clearly highlight the trade-offs inherent to boosting variational inference.
Empirically, our method enjoys improved performance over line-search while requiring significantly less memory consumption than a predefined step-size.

Applying this work to more complex inference problems such as Latent Dirichlet Analysis (LDA) is a promising direction for future work.
\newpage

\bibliographystyle{named}
\bibliography{bibliography}

\onecolumn
\appendix

\section{Corrective Variants}\label{app:fwvariants}

In the so-called \emph{away-steps} variant, the algorithm can decide to reduce the weight of a component rather than add a new one if that yields more progress. 
If $\innerp{\ln\frac{p_X}{q_t}, s_t - q_t}\geq \innerp{\ln\frac{p_X}{q_t}, q_t - \overline{v}}$ then the update direction ($v_t$ at line~(5) of Alg.~\ref{alg:main}) is set to $q_t$. This is a normal update where we add a new component to the mixture. If $q_t - \overline{v}$ is more aligned with the negative gradient, then the algorithm performs a \emph{downweighting} step and sets $v_t = s_t - q_t + \overline{v}$ thereby resulting in the desired proposal density, $q_t - \overline{v}$, when being subtracted from $s_t$ (line 6, Alg.~\ref{alg:main}). In this case, no new component is added and we simply decrease the weight of an existing component.

The so-called \emph{pair-wise} variant consistently sets the update direction as $v_t = \overline{v}$.

For both away-steps and pair-wise, we need to ensure that the weights never become negative.
In the case of a downweighing step for away-steps, we set $\gammamax_t = \alpha_{\overline{v}}/(1-\alpha_{\overline{v}})$ where $\alpha_{\overline{v}}$ is the weight corresponding to $\overline{v}$.
Otherwise, away-steps performs a normal update and $\gammamax_t=1$.
For pair-wise, weight is shifted from $\overline{v}$ to $s_t$. Thus, the maximum weight that can be shifted is $\gammamax_t = \alpha_{\overline{v}}$.

\section{Proofs}\label{app:proof}

\subsection{Definitions}
\begin{enumerate}
\item
  $\D\subset \mathcal H$ a convex set
\item
  $\psi(u) \defeq -\dual{f}(u) - \dual{\indicator_\D}(-u)$, the dual objective where,
  \begin{enumerate}
  \item
    $\indicator_\D(u) \defeq 0$ when $u\in\domain$ and $\indicator_\D(u) \defeq +\infty$ when $x\notin\domain$
  \item
    $\dual{f}(u) := \max_{s\in\mathcal H} \innerp{u,s} - f(s)$
  \end{enumerate}
\item
  $h_t \defeq f(q_t) - f(q^{*})$, the suboptimality at step $t$
\end{enumerate}

\subsection{Known facts}

Here we repeat a number of known facts that we use in the proof. Proofs of these results can be found in standard textbooks such as \citet{nesterov2018lectures,Boyd:2004uz}.

\begin{remark}\label{remark:dualitygap} Frank-Wolfe gap (c.f.\ Lemma 8 of \citet{Pedregosa2018Jun})
  \begin{equation}
    \innerp{\nabla f(q_t), x_t - s}
    \geq \delta (f(q_t) - \psi(\nabla f(q_t)))
  \end{equation}
\end{remark}

\hide{
\begin{proof}
  \begin{align}
    f(q_t) - \psi(\nabla f(q_t))
    &= f(q_t) - (-\dual{f}(\nabla f(q_t)) - \dual{\indicator_\D}(-\nabla f(q_t))\\
    &= f(q_t) + \dual{f}(\nabla f(q_t)) + \dual{\indicator_\D}(-\nabla f(q_t)\\
    &= \innerp{q_t, \nabla f(q_t)} + \max_{s\in\mathcal H} \{ \innerp{-\nabla f(q_t), s} - \indicator_\D(s) \}\label{eq:22}\\
    &= \innerp{q_t, \nabla f(q_t)} + \max_{s\in\D} \{ \innerp{-\nabla f(q_t), s} \}\\
    &= \max_{s\in\D}\innerp{\nabla f(q_t), q_t - s} \\
    &= g_t
  \end{align}
  Equation~\eqref{eq:22} follows from

  \begin{equation}
    f(q_t) + \max_{s\in\mathcal H}\innerp{\nabla f(q_t), s} - f(s) = f(q_t) + (\innerp{\nabla f(q_t), q_t} - f(q_t)) = \innerp{\nabla f(q_t), q_t}
  \end{equation}
\end{proof}
}

\begin{remark}
  Zero duality gap. Namely,
\begin{equation}
  \min_{q\in\D} f(q) = \max_{u\in\mathcal H} \psi(u)
\end{equation}
\end{remark}

\hide{
\begin{proof}
  \begin{align}
    \min_{q\in\mathcal D} f(q)
    &= \min_{q\in\mathcal H} f(q) + \indicator_D(q)\\
    &= \min_{q\in\mathcal H} \max_{u\in\mathcal H} \innerp{q,u} - \dual{f}(u) + \indicator_D(q)\label{eq:dualdual}\\
    &= \max_{u\in\mathcal H} \min_{q\in\mathcal H} \innerp{q,u} - \dual{f}(u) + \indicator_D(q)\label{eq:sion}\\
    &= \max_{u\in\mathcal H} \max_{q\in\mathcal H} \innerp{-q,u} + \indicator_\D(-q) - \dual{f}(u)\\
    &= \max_{u\in\mathcal H} \max_{q\in\mathcal H} \innerp{q,-u} + \indicator_\D(-q) - \dual{f}(u)\\
    &= \max_{u\in\mathcal H} -\dual{\indicator_\D}(-u) - \dual{f}(u)\\
    &= \max_{u\in\mathcal H} \psi(u)
  \end{align}
Equation \eqref{eq:dualdual} follows from $f^{**} = f$.
Equation \eqref{eq:sion} follows from Sion's minimax theorem.
\todogideon{Show that $f^{**} = f$.}
\todogideon{Explain why Sion's minimax applies here. Need to show convexity/concavity of the functions.}
\end{proof}
}

\begin{remark}
  The previous remark implies that
  \begin{equation}
    h_t \defeq f(q) - f(q^{*}) \leq f(q) - \psi(u) \quad\forall u\in\mathcal H
  \end{equation}
\end{remark}

\begin{remark}
  $\CfQ$ is bounded if the parameter space of $\overline{Q}$ is bounded (c.f.\ Thm.\ 2 of \citet{locatello2018boosting})
\end{remark}

\subsection{Consequences of Boundedness on Monte-Carlo Errors Assumption}\label{sec:assumption-consequences}

For convenience, we repeat the assumption stated in Equation~\eqref{eq:assumption}: for all $\varepsilon > 0$ there exists an $n\in\mathbb N$ such that for all $s\in\mathcal Q$ and $q\in\conv(\mathcal Q)$,
\begin{equation}
    | \innerp{s,\grad\dkl(q)} - \reallywidehat{\innerp{s,\grad\dkl(q)}}| \leq \varepsilon
\end{equation}

We now state a couple consequences of this assumption as lemmas.

\begin{lemma}\label{lemma:finite_mixture_bound}
For any finite mixture $q_t = \sum_{i=1}^t \alpha_i s_i$, where $\alpha$ is in the $t$-simplex, i.e.\ $\sum_{i=1}^t \alpha_i = 1$ and $\alpha_i\geq 0$, and $s_i\in\mathcal Q$,
\begin{equation}
|\dkl(q_t) - \widehat\dkl(q_t)| \leq \varepsilon
\end{equation}
\end{lemma}
\begin{proof}
First, note that 
\begin{equation}
    \dkl(q_t) = \innerp{q_t,\grad \dkl(q_t)}
\end{equation}
Next, let $\alpha \in (0,1)$. Then,
\begin{equation}
    | \innerp{\alpha s, \grad\dkl(q)} - \reallywidehat{\innerp{\alpha s, \grad\dkl(q)}} |
     = \alpha
    | \innerp{s, \grad\dkl(q)} - \reallywidehat{\innerp{s, \grad\dkl(q)}} | \leq \alpha\varepsilon
\end{equation}

Now it is convenient to make a slight abuse of notation. Let $\widehat\dkl(q_t) = \frac{1}{n}\sum_{i=1}^n \frac{q_t(\sigma_i)}{p_X(\sigma_i)}$ where $\sigma_i\sim q_t$.
\begin{align}
    \widehat\dkl(q_t) &= \reallywidehat{\innerp{q_t,\grad\dkl(q_t)}}\\
    &=
    \reallywidehat{\innerp{\sum_{i=1}^t \alpha_i s_i,\grad\dkl(q_t)}}\\
    &=
    \sum_{i=1}^t \alpha_i \reallywidehat{\innerp{s_i,\grad\dkl(q_t)}}
\end{align}
The last equality that the samples from the mixture $q_t$ can divided into the samples that come from each of its components. Also, the number of samples is different for each approximation, 
$\reallywidehat{\innerp{s_i,\grad\dkl(q_t)}} = \sum_{j=1}^{n_i}\frac{q_t(\sigma_{i,j})}{p_X(\sigma_{i,j})}$ where $\sigma_{i,j}\sim s_i$ and $\sum_{i=1}^n n_i = n$. With this abuse of notation, it is easy to express the following

\begin{align}
|\dkl(q_t) - \widehat\dkl(q_t)| &= 
|\innerp{q_t,\grad\dkl(q_t)} - \reallywidehat{\innerp{q_t,\grad\dkl(q_t)}}|\\
&= | \sum_{i=1}^t \alpha_i(\innerp{s_i,\grad\dkl(q_t)} - 
\reallywidehat{\innerp{s_i,\grad\dkl(q_t)}}) | \\
&\leq
 \sum_{i=1}^t \alpha_i|\innerp{s_i,\grad\dkl(q_t)} - 
\reallywidehat{\innerp{s_i,\grad\dkl(q_t)}} | \\
&\leq \sum_{i=1}^t \alpha_i \varepsilon = \varepsilon
\end{align}

The first inequality is the triangle inequality. The second follows from our key assumption (see above).

\end{proof}

\begin{lemma}
Let $Q_t$ and $\widehat Q_t$ be as defined in Equations~\ref{eq:defQt}~and~\ref{eq:defQt-hat} in Section~\ref{sec:local-boosting}. Then,
  \begin{align}
    |Q_t(\gamma, C) - \widehat Q_t(\gamma,C)| &= 
    (1-\gamma)(\dkl(q_t) - \widehat\dkl(q_t))
    + \gamma(
    \innerp{\grad\dkl(q_t),s_t}
    - \reallywidehat{\innerp{\grad\dkl(q_t),s_t}}
    )
    + 0\\
    &\leq
    (1-\gamma) \varepsilon + \gamma \varepsilon = \varepsilon
\end{align}
\end{lemma}

\begin{proof}
Observe,
\begin{align}
    Q_t(\gamma,C) &= \dkl(q_t) + \gamma \innerp{\grad \dkl(q_t), s_t - q_t} + \frac{C}{2}\gamma^2\\
    &= 
    \dkl(q_t) + \gamma \innerp{\grad\dkl(q_t),s_t} - \gamma \innerp{\grad\dkl(q_t),q_t} + \frac{C}{2}\gamma^2\\
    &= (1-\gamma)\dkl(q_t) + \gamma\innerp{\grad\dkl(q_t),s_t} + \frac{C}{2}\gamma^2
\end{align}
Thus, 
\begin{align}
    |Q_t(\gamma, C) - \widehat Q_t(\gamma,C)| &= 
    (1-\gamma)(\dkl(q_t) - \widehat\dkl(q_t))
    + \gamma(
    \innerp{\grad\dkl(q_t),s_t}
    - \reallywidehat{\innerp{\grad\dkl(q_t),s_t}}
    )
    + 0\\
    &\leq
    (1-\gamma) \varepsilon + \gamma \varepsilon = \varepsilon
\end{align}

The inequality follows from Lemma~\ref{lemma:finite_mixture_bound}.
\end{proof}

\subsection{Approximate backtracking}

\begin{lemma}\label{lemma:optgammat}
  For $q_{t+1} = q_t + \gamma_t d_t$ with $\gamma_t$ given by approximate backtracking and $d_t \defeq s_t - v_t$, $v_t$ is defined in Algorithm~\ref{alg:main},
  we have
\begin{equation}
  \dkl(q_{t+1}) \leq \widehat Q_t'(\xi, C_t) \quad\forall \xi\in [0, 1]
\end{equation}
\end{lemma}

\begin{proof}
  By the optimality of $\gamma_t$ we have that
  \begin{equation}
    \widehat Q_t'(\gamma_t, C_t)
    \leq
    \widehat Q_t'(\xi, C_t)
    \quad\forall\xi\in[0,\gammamax_t]
  \end{equation}
  Combining these two inequalities proves the lemma.
\end{proof}

\begin{proposition}[The curvature estimate is upper bounded]\label{prop:boundedLt}
  If the curvature estimate is initialized such that $C_{-1} \leq \Cf $, then for all $t$ we have that
\begin{equation}
    C_t\leq \tau \Cf~.
\end{equation}
\end{proposition}

\begin{proof}
  The curvature estimate increases by a factor of $\tau$ each time the sufficient condition is violated, hence the only way that a curvature estimate $C$ could be larger than $\tau \Cf$ is if $C \geq \Cf$ and did not verify the sufficient decrease condition, which is impossible by the definition of sufficient decrease.
\end{proof}

\subsection{Main result}

We now provide a proof for the simple variant of our boosting variational inference algorithm (Alg.~\ref{alg:main}) with approximate backtracking (Alg.~\ref{alg:approxback}). This can trivially be extended to the away-steps and pair-wise variants using the same techniques described in \citet{Pedregosa2018Jun}.

\begin{reptheorem}{thm:convergence}
  Let $q_t$ be the $t$-th iterate generated by Algorithm~\ref{alg:main}. Let $\varepsilon_t = \frac{\varepsilon_0}{t^2}$ bound the Monte-Carlo approximation error, described in Equation~\eqref{eq:assumption}, with some initial approximation error $\varepsilon_0>0$.
  Let $\overline C_t \defeq \frac{1}{t}\sum_{i=0}^{t-1} C_t$ be the average of the curvature constant estimates.
  Then we have:
  \begin{align*}
    \dkl(q_t) - \dkl(q^\ast) \leq 
    \frac{4(1-\delta)}{t\delta(t\delta+1)}E_0
    + \frac{2\overline{C_t}}{\delta(t\delta + 1)}
    + \frac{2\varepsilon_0}{t}
  \end{align*}
  where $E_0 \defeq \dkl(q_0) - \psi(\nabla \dkl(q_0))$ is the initialization error ($\psi$ denotes the dual objective) and $\delta>0$ bounds the error of estimating the greedy subproblem defined in Equation~\eqref{eq:bbbvi}. 
\end{reptheorem}

\begin{proof}
By Lemma~\ref{lemma:optgammat}, we know that for all $\xi_t\in[0,1]$,
\begin{align}
\dkl(q_{t+1}) &\leq \dkl(q_t) - \xi_t \innerp{\nabla \dkl(q_t), q_t-s_t} + \frac{C_t\xi_t^2}{2} + \frac{2\varepsilon_0}{t^2}\\
&\leq
\dkl(q_t) -\xi_t\delta(\dkl(q_t) - \psi(\nabla \dkl(q_t))) + \frac{C_t\xi_t^2}{2} + \frac{2\varepsilon_0}{t^2}\label{eq:dualexpand}\\
&=(1-\xi_t\delta)\dkl(q_t) + \xi_t\delta\psi(\nabla \dkl(q_t)) + \frac{C_t\xi_t^2}{2} + \frac{2\varepsilon_0}{t^2}
\end{align}

Now we define a sequence recursively as $\sigma_{t+1} \defeq (1-\xi_t\delta)\sigma_t + \xi_t\delta\psi(\nabla \dkl(q_t))$ and $\sigma_0 \defeq \psi(\nabla \dkl(q_0))$.
Then if we subtract $\sigma_{t+1}$ from both sides we get,
\begin{align}
\dkl(q_{t+1}) - \sigma_{t+1}
&\leq
(1-\xi_t\delta)\dkl(q_t) + \xi_t\delta\psi(\nabla \dkl(q_t)) - ((1-\xi_t\delta)\sigma_t + \xi_t\delta\psi(\nabla \dkl(x_t)))
+ \frac{C_t\xi_t^2}{2} + \frac{2\varepsilon_0}{t^2}\\
&=
(1-\xi_t\delta)(\dkl(q_t) - \sigma_t)
+ \frac{C_t\xi_t^2}{2} + \frac{2\varepsilon_0}{t^2}\label{eq:28}
\end{align}
Recall that this inequality is valid true for all $\xi_t\in[0,1]$.
In particular, its valid for $\xi_t = \nicefrac{2}{(\delta t+2)}$.

Now define $a_t \defeq \frac{1}{2} ((t-2)\delta + 2)((t-1)\delta + 2)$.
Note that,
\begin{align}
a_{t+1}(1 - \xi_t\delta) &= \left(\frac{1}{2}((t-1)\delta + 2)(t\delta + 2)\right) \left( \frac{\delta t + 2 - 2\delta}{\delta t + 2} \right)\\
&= \frac{1}{2} ((t-1)\delta + 2) (\delta t + 2 - 2\delta) \\
&= \frac{1}{2} ((t-1)\delta + 2) ((t-2)\delta + 2)\\
&= a_t
\end{align}

Multiply both sides of the inequality in Equation~\eqref{eq:28} by $a_{t+1}$. This gives
\begin{align}
a_{t+1}(\dkl(q_{t+1}) - \sigma_{t+1})
&\leq
a_t(\dkl(q_t) - \sigma_t)
+ a_{t+1}\left(\frac{C_t\xi_t^2}{2} + \frac{2\varepsilon_0}{t^2}\right)
\end{align}

We can upper bound this further simply by noting that,
\begin{align}
a_{t+1} \xi_t^2 &= \frac{1}{2}((t-1)\delta + 2)(t\delta + 2)\left(\frac{2}{\delta t + 2}\right)^2\\
                &= \frac{1}{2}((t-1)\delta + 2)\left(\frac{4}{\delta t + 2}\right) \\
                &= 2\frac{\delta(t-1) + 2}{\delta t + 2} \\
                &\leq 2
\end{align}
Thus,
\begin{align}
a_{t+1}(\dkl(q_{t+1}) - \sigma_{t+1})
\leq
a_t(\dkl(q_t) - \sigma_t)
+ C_t + 2a_{t+1}\frac{\varepsilon_0}{t^2}
\end{align}

Let us now unroll this inequality starting at step $t$,
\begin{align}
a_{t}(\dkl(q_{t}) - \sigma_{t})
&\leq
a_0(\dkl(q_0) - \sigma_0)
+ \sum_{i=0}^{t-1} (C_i + 2a_{i+1}\frac{\varepsilon_0}{t^2})\\
&=
a_0(\dkl(q_0) - \sigma_0)
+ t\overline C_t + \frac{2\varepsilon_0}{t^2} \sum_{i=0}^{t-1}  a_{i+1}
\label{eq:in_terms_of_at}
\end{align}
where $\overline C_t\defeq \frac{1}{t}\sum_{i=0}^{t-1} C_t$.

Since $\delta \in (0,1]$, we can make a simple observation:
\begin{align}
&a_t \defeq \frac{1}{2} ((t-2)\delta + 2)((t-1)\delta + 2) \geq \frac{1}{2} (t\delta)(t\delta+1)\\
\implies
&\frac{1}{a_t} \leq \frac{2}{(t\delta)(t\delta+1)}
\end{align}
Now divide both sides of Equation~\eqref{eq:in_terms_of_at} by $a_t$ and apply this observation,
\begin{align}
\dkl(q_{t}) - \sigma_{t}
\leq
\frac{a_0}{a_t}(\dkl(q_0) - \sigma_0)
+ \frac{2\cancel{t}}{(\cancel{t}\delta)(t\delta+1)}\overline C_t + \frac{2\varepsilon_0}{t^2} \frac{1}{a_t} \sum_{i=0}^{t-1}  a_{i+1}\label{eq:summation_ai}
\end{align}
Let us focus on the summation term,
\begin{align}
    \sum_{i=0}^{t-1} a_{i+1}
    &= \frac{1}{2}\sum_{i=0}^{t-1}
     ((i-1)\delta + 2)(i\delta + 2)\\
    &\leq
    \frac{1}{2}\sum_{i=0}^{t-1}
     ((t-2)\delta + 2)((t-1)\delta + 2)\\
     &=
    \frac{1}{2}t
     ((t-2)\delta + 2)((t-1)\delta + 2)
\end{align}
Thus,
\begin{align}
\frac{1}{a_t}\sum_{i=0}^{t-1} a_{i+1}
&\leq
\frac{{2}}{{((t-2)\delta + 2)((t-1)\delta + 2)}}
{\frac{1}{2}}t
{((t-2)\delta + 2)((t-1)\delta + 2)}\\
&= t
\end{align}
Now we can use this to substitute into Equation~\eqref{eq:summation_ai},
\begin{align}
\dkl(q_t) - \sigma_t
\leq
\frac{a_0}{a_t}(\dkl(q_0) - \sigma_0)
+ \frac{2}{\delta(t\delta + 1)} \overline{C_t}
+ \frac{2\varepsilon_0}{t}
\end{align}

We now bound the ratio $a_0/a_t$. First note that the factor of $1/2$ cancels. We use the fact that $\delta\in (0,1]$ to bound the ratio. Notice that $2a_0 = (2-2\delta)(2-\delta)\leq 2(2-2\delta)$. Finally,
\begin{equation}
    2a_t = ((t-2)\delta + 2)((t-1)\delta + 2
    =
    (t\delta - 2(\delta-1))(t\delta + 1 - (\delta -1))\geq t\delta(t\delta+1)
\end{equation}
This gives,
\begin{equation}
    \dkl(q_t) - \sigma_t \leq \frac{4(1-\delta)}{t\delta(t\delta+1)}(\dkl(q_0) - \sigma_0)
    + \frac{2}{\delta(t\delta + 1)} \overline{C_t}
    + \frac{2\varepsilon_0}{t}
\end{equation}
By Proposition~\ref{prop:boundedLt}, $\overline{C_t}$ is bounded.

We now show that $f(q^{t}) - \psi(u_t) \leq f(q^{t}) - \sigma_t$.
Or equivalently, that $-\psi(u_t) \leq -\sigma_t$.
We do this by induction.
The base case is trivial since $\sigma_0 = \psi(\nabla f(q^0)) = \psi(u_0)$.
Assume it is true at step $t$. We now prove it holds for step $t+1$.
\begin{align}
-\psi(u_{t+1})
  &=
    -\psi((1-\xi_t)u_{t}+\xi_t\nabla f(q_t))\\
  &\leq
    -(1-\xi_t)\psi(u_{t})-\xi_t\psi(\nabla f(q_t))\label{eq-convergence:convexity}\\
  &\leq
    -(1-\xi_t)\sigma_{t}-\xi_t\psi(\nabla f(q_t))\label{eq-convergence:inductassum}\\
  &=
    -\sigma_{t+1}
\end{align}
Equation \eqref{eq-convergence:convexity} follows by convexity of $-\psi$.
Equation \eqref{eq-convergence:inductassum} follows by the inductive assumption.

Putting this together gives our a non-convergent rate, note the last term:
\begin{equation}
    h_t \leq 
    \frac{4(1-\delta)}{t\delta(t\delta+1)}(\dkl(q_0) - \psi(\grad\dkl(q_0)))
    + \frac{2}{\delta(t\delta + 1)} \overline{C_t}
    + \frac{2\varepsilon_0}{t}
\end{equation}
This equation shows that to achieve a convergent rate, we must choose a decay rate for $\varepsilon$ which is $\mathcal O(1/t^2)$. Doing so provides the desired rate.
\end{proof}

\section{Detailed results}\label{app:results}
\subsection{Experimental setup}
We used the \emph{Edward} probabilistic programming framework \citep{tran2016edward} for the implementation of approximate backtracking.
This allows us to apply our algorithm to any of the probabilistic models that are definable by \emph{Edward}'s domain specific language for probabilistic models.
We extend the Residual Evidence Lower Bound (RELBO) defined in \citet{locatello2018boosting}.

In our implementation of approximate backtracking, we estimate the Lipschitz constant instead of the curvature constant.
This gives a slightly different quadratic upper bound:
\begin{equation}
  Q_t(\gamma, L) = f(q_t) + \innerp{\nabla f(q_t), s_t - q_t} + \frac{L}{2}||s_t-q_t||^2
\end{equation}
which depends on the squared distance between the next component $s_t$ and the current iterate $q_t$.
Empirically, we found that the value of the Lipschitz constant was bounded resulting in an equivalent algorithm.
This approach is closer to the work of \citet{Pedregosa2018Jun} whose analysis is also in terms of the Lipschitz constant.

In our experiments, we used the KL-divergence, $\dkl(s_t\,\|\,q_t)$, rather than the squared-distance for the quadratic term.
We found this to be more stable than the squared distance for density functions.

For most experiments, we observe that Boosting VI needs a couple of iterations to improve upon simple black box VI i.e.\ the first component.

We implemented our own version of all algorithms including those of \citep{locatello2018boosting}.
In this process, we found some discrepancy between our results and the results reported in their paper.
We report the results that we were able to reproduce.

\subsection{Early termination of adaptive loop}

\begin{figure}[t]
 \center
    \includegraphics[width=0.95\linewidth]{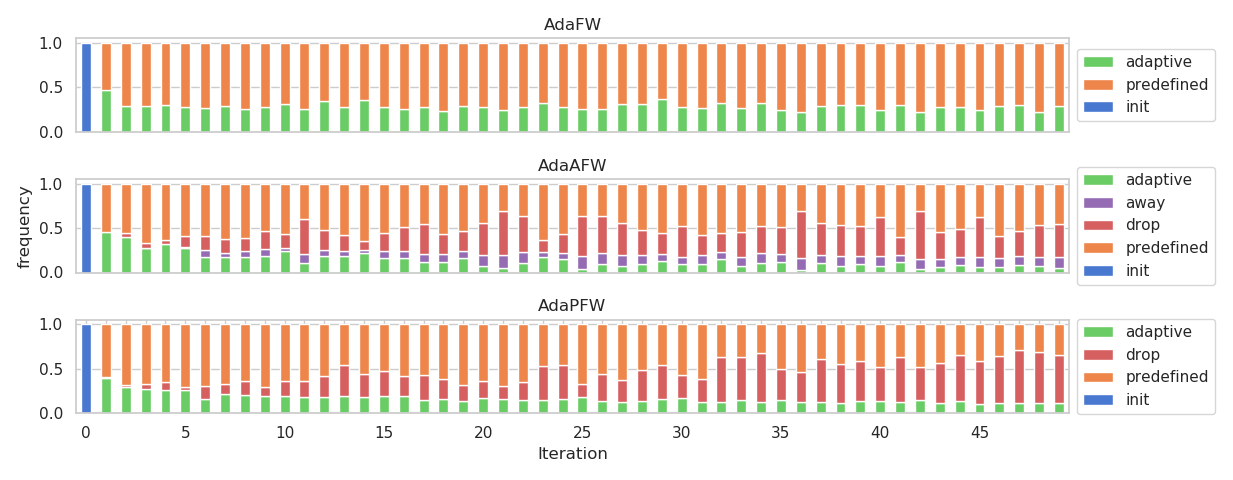}
    \includegraphics[width=0.95\linewidth]{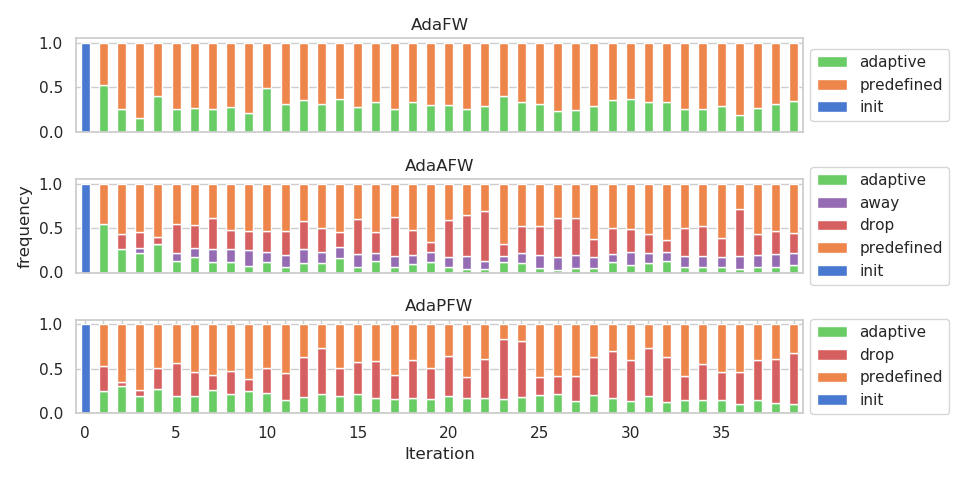}
    \caption{Frequency of different types of steps for \textsc{ChemReact} (top) and \textsc{eICU} (bottom) experiments.}
    \label{fig:iter_info_blr_and_eicu}
\end{figure}

As discussed in Section~\ref{sec:local-boosting}, our implementation of approximate backtracking includes an early termination criteria.
When the curvature is enormous, the approximate backtracking loop takes a prohibively long time to find a suitable quadratic tanget function.
To avoid these cases, we break out of the approximate backtracking loop after a fixed number of steps and perform a predefined step-size update.
In particular, we set $\textsc{imax} = 10$ in Algorithm~\ref{alg:approxback}.

In Figure~\ref{fig:iter_info_blr_and_eicu}, we show the frequency of different update steps at every iteration aggregated over all hyperparameters and ten random seeds.
This gives a total of 600 runs for \textsc{ChemReact} and 200 runs for \textsc{eICU}.

\subsection{Line-search instability}

Here we provide details for the experiment presented in the main paper in Section~\ref{subsec:syn_data}.

We implemented projected stochastic gradient for line-search on $\gamma\in[0,1]$. We treated the gradient descent step-size, denoted $b_0$, as a hyperparameter. Line-search does not take into account any smoothness or curvature values which leaves empirical tuning of $b_0$ as the only option.

We set the hyperparameters of our adaptive step-size algorithm as:
\begin{align*}
  \tau &= 2.0\\
  \frac{1}{\eta} &= 0.1\\
  L_{-1} &= 10.0
\end{align*}
For line-search we set $b_0=0.1$.

\subsection{Bayesian logistic regression}
\subsubsection{Hyperparameters}
For the \textsc{ChemReact} experiment, we selected the adaptive hyperparameters as
\begin{align*}
  \tau &\in \{1.01, 1.1, 1.5, 2.0, 5.0\}\\
  \frac{1}{\eta}&\in\{0.1, 0.01, 0.5, 0.99\}\\
  L_{-1}&\in\{0.01, 1.0, 100.0\}
\end{align*}
For line-search, the initial step-size was chosen as $b_0\in\{0.001, 0.0001, 0.0005, 0.01, 0.1, 0.05, 0.005\}$

For \textsc{eICU} dataset, we set $b_0 = 10^{-8}$ and ran the experiment over 10 replicates using the same range of $\eta, \tau$ as for \textsc{ChemReact}. We set $L_{-1} = 1.0$.

\subsubsection{Hyperparameters}

We set the initial Lipschitz estimate $L_{-1}=1.0$.
For line-search, we attempted a number of values for the initial step-size $b_0 \in \{10^{-10}, 10^{-9}, 10^{-8}, 10^{-7}, 10^{-6}\}$ according to the magnitude of the gradient but were unable to obtain a stable algorithm.
Line-search consistently returned $\gamma = 0$ for all of these values.

\subsection{Memory usage on \textsc{eICU}}

Here we present further results demonstrating the memory efficiency of our proposed corrective methods on the \textsc{eICU} dataset. These results are comparable to what is presented in Figure~\ref{fig:memusage_chemreact} of the main paper.

\begin{figure}[ht]
    \centering
    \includegraphics[width=.7\linewidth]{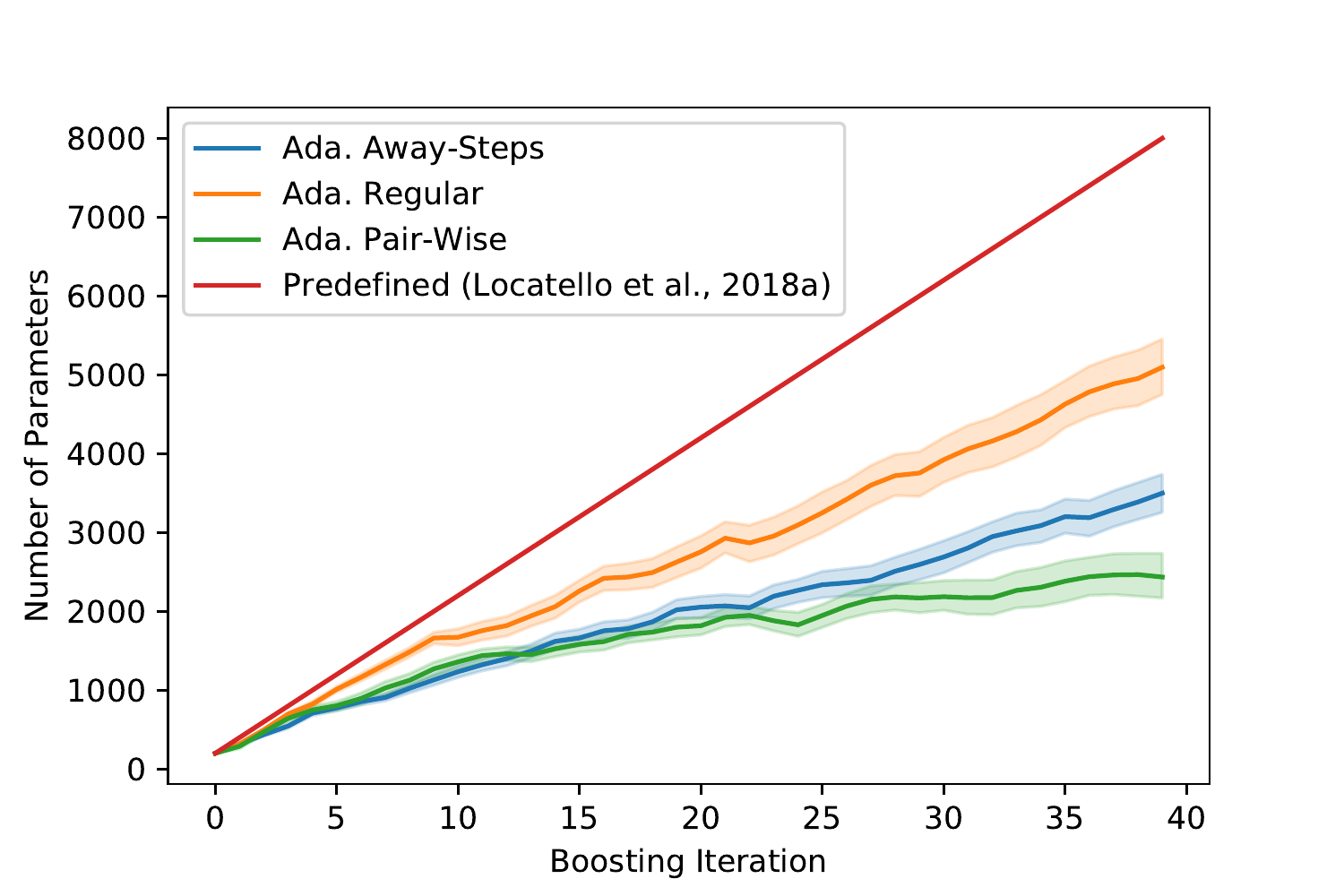}
    \caption{Comparing the number of parameters per iteration to previous work on \textsc{eICU}}
    \label{fig:memusage_eicu}
\end{figure}

\end{document}